\documentclass[12pt]{article}

\usepackage{enumerate}
\usepackage{amssymb,amsfonts,amsmath,amsthm,amsbsy}
\usepackage{latexsym}
\usepackage[mathscr]{euscript}
\usepackage{color}
\usepackage{array}
\usepackage{flafter}
\usepackage{layout}
\usepackage{graphicx}
\usepackage{epstopdf}
\usepackage[ansinew]{inputenc}
\usepackage{xfrac}
\usepackage{hyperref}
\usepackage{algorithm}
\usepackage{natbib}
\usepackage{setspace}
\usepackage[margin=1in]{geometry}

\theoremstyle{plain}
  \newtheorem{teo}{Theorem}
  
  \newtheorem{lem}{Lemma}
  \newtheorem{prop}{Proposition}
  
\theoremstyle{definition}

\theoremstyle{remark}
  \newtheorem{remark}{Remark}

\graphicspath{{./Figures/}}
\begin{document}
\title{Unsupervised Bump Hunting Using Principal Components}
\date{\today}
\author{Daniel A D\'{\i}az-Pach\'on\thanks{Ddiaz3@med.miami.edu} \and Jean-Eudes Dazard\thanks{jxd101@case.edu} \and J. Sunil Rao\thanks{JRao@biostat.med.miami.edu}}

\maketitle

\begin{abstract}
Principal Components Analysis is a widely used technique for dimension reduction and characterization of variability in multivariate populations. Our interest lies in studying when and why the rotation to principal components can be used effectively within a response-predictor set relationship in the context of mode hunting. Specifically focusing on the Patient Rule Induction Method (PRIM), we first develop a fast version of this algorithm (fastPRIM) under normality which facilitates the theoretical studies to follow.  Using basic geometrical arguments, we then demonstrate how the PC rotation of the predictor space alone can in fact generate improved mode estimators. Simulation results are used to illustrate our findings.\\
\textbf {Key words:} Algorithms, Bump hunting, Computationally intensive methods, Mode hunting, Principal components.
\end{abstract}

\section{Introduction}\label{Intro}

The PRIM algorithm for bump hunting was first developed by \citet{FriedmanFisher1999}. It is an intuitively useful computational algorithm for the detection of local maxima (or minima) on target functions. Roughly speaking, PRIM {\it peels} the (conditional) distribution of a response from the outside in, leaving at the end rectangular boxes which are supposed to contain a bump (see the formal description in Algorithm \ref{AlgoPRIM}) at page \pageref{AlgoPRIM}. However, some shortcomings against this procedure have also appeared in the literature when several dimensions are under consideration. For instance, as \citet{PolonikWang2010} explained it, the method could fail when there are two or more modes in high-dimensional settings.

Almost at the same time, \citet{DazardRao2010} proposed a supervised bump hunting strategy, given that the use of PRIM is still ``challenged in the context of high-dimensional data''. The strategy, called Local Sparse Bump Hunting (LSBH) is outlined in Algorithm \ref{AlgoLSBH} at page \pageref{NotationConcepts}. Summarizing the algorithm, it uses a recursive partitioning algorithm (CART) to identify subregions the whole space where at most one mode is estimated to be present; then a Sparse Principal Component Analysis (SPCA) is performed separately on each local partition; and finally, the location of the bump is determined via PRIM in the local, rotated and projected subspace induced by the sparse principal components.

As an example, we show in Figure \ref{figure01} simulation results representing a multivariate bimodal situation in the presence of noise, similarly to the simulation design used by \citet{DazardRao2010}. We simulated in a three-dimensional input space ($p = 3$) for visualization purposes. The data consists of a mixture of two trivariate normal distributions, taking on discrete binary response values ($Z \in \{1,2\}$), noised by a trivariate uniform distribution with a null response ($Z = 0$), so that the the data can be written by $X \sim w \cdot N_p(0,\Sigma) + (1 - w) \cdot B_p$, where $B_p \sim U_p[a,b]$, $w \in [0,1]$ is the mixing weight, and $(a,b) \in \mathbb R^2$.

\begin{figure}[!hbt]
    \centering
    \includegraphics[width=1.0\textwidth]{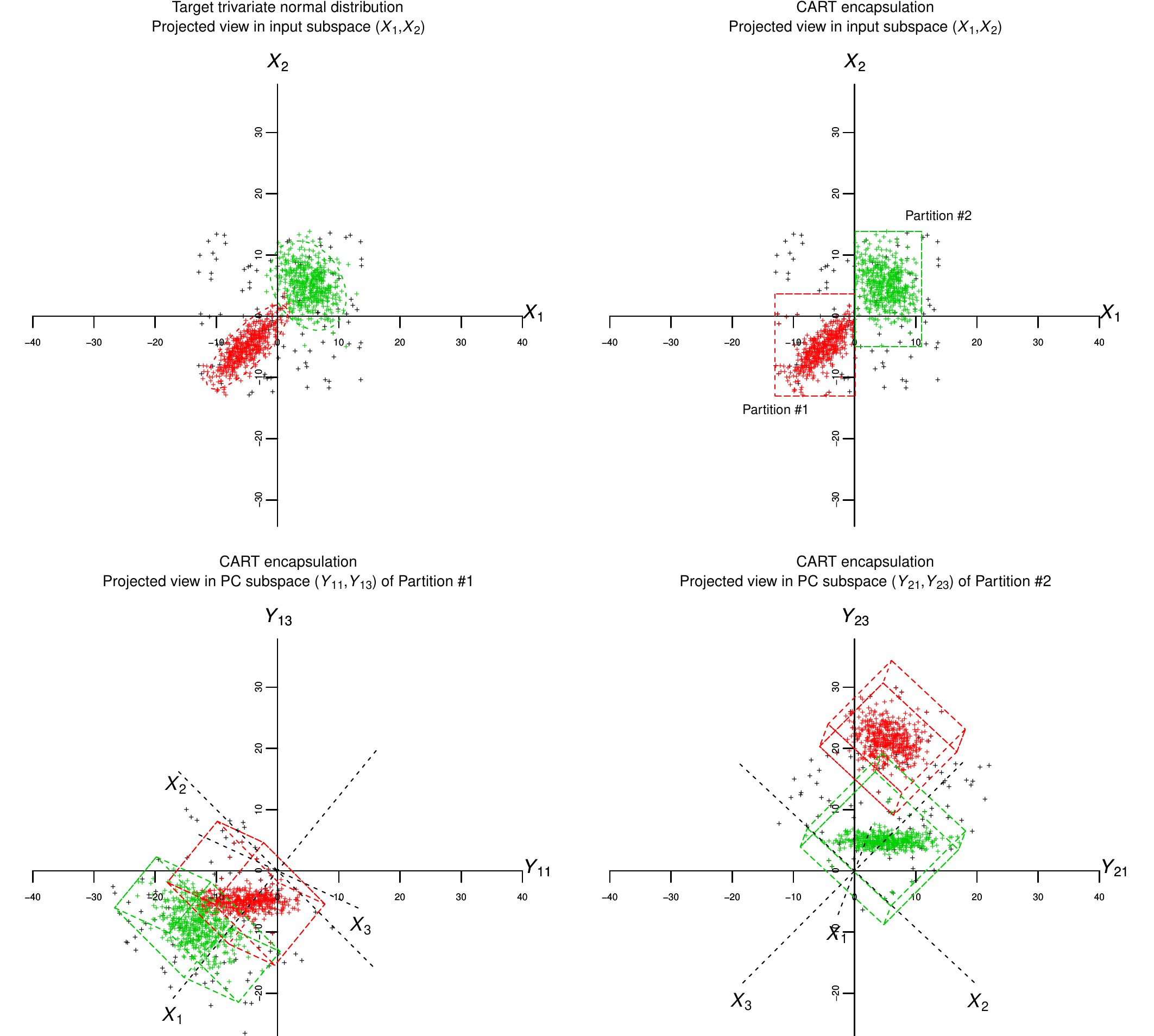}
    \vskip -10pt
    \caption{Illustration of the efficiency of the encapsulation process by LSBH of two target normal distributions (red and green dots), in the presence of 10\% ($w = 0.9$) noise distribution (black dots) in a three-dimensional input space ($p = 3$). We let the total sample size be $n = 10^3$. Top row: each plot represents a projected view of the data in input subspace ($X_{1}$, $X_{2}$) with 95\% confidence ellipses (dotted red and green contours - top left panel) and partitions vertices (top right panel). Only those partitions encapsulating the target distributions are drawn. Bottom row: each plot represents a projected view of the data in the PC subspace ($Y_{11}$, $Y_{13}$) of Partition \#1 (bottom left), and ($Y_{21}$, $Y_{23}$) of Partition \#2 (bottom right).}
    \label{figure01}
\end{figure}

Notice how the data in the PC spaces determined by Partition \#1 and \#2 do align with the PC coordinate axes $Y_{11}$ and $Y_{21}$, respectively (Figure \ref{figure01}).

Our goal in this paper is to provide some theoretical basis for the use of PCs in mode hunting using PRIM and a modified version of this algorithm that we called ``fastPRIM''.  Although the original LSBH algorithm accepts more than one mode by partition, we will restrict ourselves to the case in which there is at most one on each partition, in order to get more workable developments and more understandable results in this work. 

In Section \ref{NotationConcepts} we define the algorithms we are working with and set some useful notation. Section \ref{SecfastPRIM} proposes a modification of PRIM (called fastPRIM) for the particular case in which the bumps are modes in a setting of normal variables that allows to compare the boxes in the original space and in the rotation induced by principal components. The approach goes beyond normality and can be shown to be true for every symmetric distributions with finite second moment, and it is also an important reduction on the computational complexity since it is also useful for samples when $n\gg0$, via the central limit theorem (Subsection \ref{fastPRIMdata}). In this section we also present simulations which display the differences between considering the original space or the PC rotation for PRIM and fastPRIM. Finally, Section \ref{SpacesCompared} proves Theorem \ref{main}, a result explaining why the (volume-standardized) output box mode is higher in the PC rotation than in the original input space, a situation observed computationally by \citet{DazardRao2010} for which we give here a formal explanation. Theorem \ref{ComparingPRIMAndFastPRIM} shows that in terms of bias and variance, fastPRIM does better than PRIM. Finally, in Section \ref{Simulations} we show additional simulations relevant to the results found in Section \ref{SpacesCompared}.

\section{Notation and basic concepts}\label{NotationConcepts}

We set here the concepts that will be useful throughout the paper to define the algorithms and its modifications. Our notation on PRIM follows as a guideline the one used by \citet{PolonikWang2010}.   

Let $X$ be a $p$-dimensional real-valued random vector with distribution $F$. Let $Z$ be an integrable random variable. Let $m(x):=\textbf E[Z|X=x]$, $x\in\mathbb R^p$. Assume without loss of generality that $m(x)\geq0$.

Define $I(A):=\int_Am(x)dF(x)$, for $A\subset\mathbb R^p$. So when $A=\mathbb R^p$, then $I(A)=\textbf EZ$. We are interested in a region $C$ such that

\begin{align}\label{Goal}
	ave(C):=\frac{I(C)}{F(C)} > \rho,
\end{align}	
where $\rho=ave(\mathbb R^p)$. Note then that $ave(C)$ is just a notational convenience for the average of $Z$ given $X\in C$.

Given a box $B$ whose sides are parallel to the coordinate axes of $\mathbb R^p$, we peel small pieces of $B$ parallel to its sides and we stop peeling when what remains of the box $B$ becomes too small. Let the class of all these boxes be denoted by $\mathcal B$. Given a subset $S(X) = S \subseteq \mathbb R^p$ and a parameter $\beta\in(0,1)$, we define

\begin{align}\label{OptBox}
	B^*_\beta=\arg\max_{B\in\mathcal B}\{ave(B|S):F(B|S)=\beta\},
\end{align}
where $ave(B|S)=I(B|S)/F(B|S)$. In words, $B^*_\beta$ is the box with maximum average of $Z$ among all the boxes whose $F$-measure, conditioned to the points in the box $S$, is $\beta$. The former definitions set the stage to define Algorithm \ref{AlgoPRIM} at page \pageref{AlgoPRIM} below.

Some remarks are in order given Algorithm \ref{AlgoPRIM}:

\begin{algorithm}[!ht]
    \caption{Patient Rule Induction Method}
    \label{AlgoPRIM}
    \begin{itemize}
        \item(Peeling) Begin with $B_1=S$. For $l=1,\ldots,L-1$, where $(1-\alpha)^L=\beta$, and $\alpha\in(0,1)$, remove a subbox contained in $B_l$, chosen among $2p$ candidates given by:
        \begin{align}\label{bjk}
		b_{j1}:=\{x\in B:x_j<x_{j(\alpha)}\},\nonumber\\
		b_{j2}:=\{x\in B:x_j>x_{j(1-\alpha)}\},
        \end{align}
        where $j=1,\ldots,p$. The subbox $b^*_l$ chosen for removal gives the largest expected value of $Z$ conditional on $B_l\setminus b^*_l(X)$. That is,
        \begin{align}\label{bmin}
		b^*_l = \arg\min\left\{I\left(b_{jv}|B_l\right):j=1,\ldots,p \text{ and } v=1,2\right\}. 	
        \end{align}
        Then $B_l$ is replaced by $B_{l+1}=B_l \setminus b^*_l$ and the process is iterated as long as the current box $B_l$ be such that $F(B_l|S)\geq\beta+\alpha$.
        \item (Pasting) Alongside the $2p$ boundaries of the resulting box $B$ on the peeling part of the algorithm we look for a box $b^+\subset S\setminus B$ such that $F(b^+|S)=\alpha F(B|S)$ and $ave((B\cup b^+)\cap S) > ave(B\cap S)$. If there exists such a box $b^+$, we replace $B$ by $(B\cup b^+)$. If there exists more than one box satisfying that condition, we replace $B$ by the one that maximizes the average $ave((B\cup b^+)\cap S)$. In words, pasting is an enlargement on the Lebesgue measure of the box which is also an enlargement on the average $ave((B\cup b^+) \cap S)$.
        \item (Covering) After the first application of the peeling-pasting process, we update $S$ by $S \setminus B_1$, where $B_1$ is the box found after pasting, and iterate the peeling-pasting process replacing $S = S^{(1)}$ by $S^{(2)} = S^{(1)} \setminus B_1$, and so on, removing at each step $k=1,\ldots,t$ the optimal box of the previous step: $S^{(k)} = S^{(k-1)} \setminus B_{k-1}$, so that $S^{(k)} = S^{(1)} \setminus \cup_{1 \leq b \leq k-1}B_b$. At the end of the PRIM algorithm we are left with a region, shaped as a rectangular box:
            \begin{align}\label{FinalRegion}
		      R_\rho(p,k)=\bigcup_{ave\left(B_k | S^{(k)}\right)\geq\rho}\left\{B_k | S^{(k)}\right\}.
            \end{align}
    \end{itemize}
\end{algorithm}

\begin{remark}
	The value $\alpha$ is the second tuning parameter and $x_{j(\alpha)}$ is the $\alpha$-quantile of $F_j(\cdot|B_l)$, the marginal conditional 		distribution function of $X_j$ given the occurrence of $B_l$. Thus, by construction,
	\begin{align}\label{marginal}
		\alpha=F_j\left(b_{jv}|B_l\right)=F\left(b_{jv}|B_l\right).
	\end{align}
\end{remark}

\begin{remark}
    	Conditioning on an event, say $\tilde A$, is equivalent to conditioning on the random variable $\textbf 1\{x\in\tilde A\}$; i.e., when this occurs, as in 	(\ref{OptBox}), we are conditioning on a Bernoulli random variable.
\end{remark}

\begin{remark}
	When dealing with a sample, we define analogs of the terms used previously and replace those terms in Algorithm \ref{AlgoPRIM} with:
	\begin{align*}
		& I_n(C)=\frac{1}{n}\sum_{i=1}^nZ_i\textbf 1\{X_i\in C\},\\
		& F_n(C)=\frac{1}{n}\sum_{i=1}^n\textbf 1\{X_i\in C\},\\
		& ave_n(C)=\frac{I_n(C)}{F_n(C)},
	\end{align*}
	where $F_n$ is the empirical cumulative distribution of $X_1,\ldots,X_n$.
\end{remark}

\begin{remark}\label{FinalBoxProbability}
	Ignore the pasting stage, considering only peeling and covering. Let us call $\beta_T$ the probability of the final region. Then
	\begin{align*}
		\beta_T= \textbf P[x\in R_\rho(p)]&=\sum_{k=1}^t\beta(1-\beta)^{k-1}\\
		&=1-(1-\beta)^t.
	\end{align*}
\end{remark}

\begin{algorithm}[!ht]
    \caption{Local Sparse Bump Hunting}
    \label{AlgoLSBH}
	\begin{itemize}
        \item Partition the input space into $R$ partitions $P_1,\ldots, P_R$, using a tree-based algorithm like CART, in such a way that there is at most one mode in each of the partitions.
        \item For $r$ from 1 to $\tilde r$
        \begin{itemize}
            \item If $P_r$ is elected for bump hunting (i.e.; if $G_r$, the number of class labels in $P_r$, is greater than 1)
            \begin{itemize}
                \item Run a local SPCA in the partition $P_r$, rotating and reducing the space to $p'\ (\leq p$) dimensions, and if possible, decorrelating the sparse principal components (SPC). Call this resulting space $\mathcal T(P_r)$.
                \item Estimate PRIM meta-parameters $\alpha$ and $\beta$ in $\mathcal T(P_r)$.
                \item Run a local and tuned PRIM-based bump hunting within $\mathcal T(P_r)$ to get descriptive rules of the bumps in the SPC space of the form $R_\rho^{(r)}(p')$, as in (\ref{FinalRegion}), where $r$ indicates the partition being considered.
                \item Rotate the local rules $R^{(r)}$ back into the input space to get rules in terms of the sparse linear combinations.
            \end{itemize}
            \item Actualize $r$ to $r+1$.
        \end{itemize}
        \item Collect the rules from all partitions to get a global rule $\mathcal R=\bigcup_{r=1}^RR_\rho^{(r)}$ giving a full description of the estimated bumps in the entire input space.
    \end{itemize}
\end{algorithm}

\subsection{Principal Components}\label{SecPRIMPC}

The theory about PCA is widely known, however we will oultine it here for the sake of completeness and to define notation. Among others, \citet{Mardia1976} presents a thorough analysis.

If $\textbf x$ is a random centered vector with covariance matrix $\Sigma$, we can define a linear transformation $\textbf T$ such that

\begin{align}\label{PopPCA}
	\textbf T\textbf x = \textbf y = \Gamma'\textbf x,
\end{align}
where $\Gamma$ is a matrix such that its columns are the standardized eigenvectors of $\Sigma:=\Gamma\Lambda\Gamma'$; $\Lambda$ is a diagonal matrix with $\lambda_1\geq\cdots\lambda_p\geq0$; and $\lambda_j$, $j=1,\ldots,p$, are the eigenvalues of $\Sigma$. Then $\textbf T$ is called the principal components transformation.

Let $p'\leq p$. We call $\mathfrak X(p)$  the original $p$-dimensional space where $\textbf x$ lives, $\mathfrak X'(p)$ the rotated $p$-dimensional space where $\textbf y$ lives, and $\mathfrak X'(p')$ the rotated and projected space on the $p'$ first PC's.

As we will explain later, we are not advising on the reduction of dimensionality in the context of regression or other learning settings. However, since it is relevant to some features of our simulations, we consider the case $\mathfrak X'(p')$ with $p'\leq p$.

\section{fastPRIM: a More Efficient Approach to mode hunting}\label{SecfastPRIM}

Despite successful applications in many fields, PRIM presents some shortcomings. For instance, \citet{FriedmanFisher1999}, the proponents of the algorithm, show that in the presence of high collinearity or high correlation PRIM is likely to behave poorly. This is also true when there is significant background noise. Further, PRIM becomes computationally expensive in simulations and real data sets in large dimensions. In this section we propose a modified version of PRIM, called ``fastPRIM'', aimed to solve these two problems when we are hunting the mode. The high collinearity problem can be solved via principal components. The computational problems can be solved via the CLT and the geometric properties of the normal distribution, if we can warrant $n\gg0$.

The following situations are variations from simple to complex of the input $X$ and the response $Z$ being normally distributed $N(\textbf 0,\Sigma)$ and $N(0,\sigma)$, respectively. We are interested on maximizing the density of $Z$ given $X$. But there are several ways to define the mode of a continuum distribution. So for simplicity, let us define the mode of $Z$ as the region $C\subset\mathbb R^p$ with $P_X[x\in C]=\beta$ that maximizes
\begin{equation}\label{Mode}
	M(C):=\int_Cf_Z(x)dF(x)
\end{equation}
(note the similarity of $M(C)$ with $I(C)$ in Equation (\ref{Goal})). In terms of PRIM, we are interested in the box $B^*_\beta$ defined on Equation (\ref{OptBox}). That is, $B^*_\beta$ is a box such that $\textbf P_X[x\in B^*_\beta]=\beta$, and inside it the mean density of the response $Z$ is maximized. Then, since the mean and the mode of the normal distribution coincide, finding a box of size $\beta$ centered around the mean of $X$ is equivalent to finding a box that maximizes the mode of $Z$ (since $X$ and $Z$ are both centered around the origin).

Although it is good to have explicit knowledge of our final region of interest, on what follows most of the results ---with the exception of Theorem \ref{main} below--- can be stated without direct reference to the mode of $Z$, taking into account that the mode of $Z$ is centered around the mean of $X$. 

\subsection{fastPRIM for Standard Normality}\label{SecPRIMStandard}

Let $X\sim N(\textbf 0,\textbf I)$ with $X$ living in the space $S(X)$. Let $Z\sim N(0,1)$. Since the whole input space is defined by symmetric uncorrelated variables, PRIM can be modified in a very efficient way. (See below Algorithm \ref{AlgofastPRIM}.)

\begin{algorithm}[!ht]
    \caption{fastPRIM with Standard Normal Predictors}
    \label{AlgofastPRIM}
    \begin{itemize}
        \item (Peeling) Instead of peeling just one side of probability $\alpha$, make $2p$ peels corresponding to each side of the box, giving to each one a probability $\alpha(2p)^{-1}$. Then, after $L$ steps, the remaining box has the same $\beta$ measure, it is still centered at the origin and its marginals will have probability measure $\beta^{1/p}$.
        \item (Covering) Call $B_M(k)$ the box found after the $k$-th step, $k=1,\ldots,t$ of this modified peeling stage. Setting $S(X)=S^{(1)}(X)$, take the space $S^{(k)}(X) := S^{(1)}(X) \setminus \bigcup_{1 \leq b \leq k-1} B_M(b)$ and repeat on it the peeling stage.
    \end{itemize}
\end{algorithm}

Several comments are worthy to mention related to this modification.

\begin{enumerate}
    \item Given that the standard normal is spherical, the final box at the end of the peeling algorithm is centered. It is also squared in that all its marginals have the same Lebesgue measure and the same probability measure $\beta^{1/p}$. Then, instead of doing the whole peeling stage, we can reduce it to select the central box whose vertices are located at the coordinates corresponding to the quantiles $\frac{1}{2}\beta^{1/p}$ and $1-\frac{1}{2}\beta^{1/p}$ of each marginal.
    \item Say we want to apply $t$ steps of covering. Since the boxes chosen are centered at the end of the $t$-th covering step, the final box will have probability measure $\beta_T:=1-(1-\beta)^t$ (which, by Remark \ref{FinalBoxProbability}, produces the same probability than PRIM), each marginal has measure ($\beta_T)^{1/p}$, and the vertices of each marginal are located at the coordinates corresponding to the quantiles $\frac{1}{2}(\beta_T)^{1/p}$ and $1-\frac{1}{2}(\beta_T)^{1/p}$. It means that the whole fastPRIM is reduced to calculating this central box of probability measure $t\beta$.
    \item The only non-zero values outside the diagonal in the covariance matrix of $(Z\; X)^T$ of size $(p+1)\times(p+1)$ are possibly the non-diagonal terms in the first row and the first column. Let us call them $\sigma_{ZX_1},\ldots,\sigma_{ZX_p}$. From this we get that $\textbf E[Z|X]=\sum_{j=1}^p\sigma_{ZX_j}X_i$ and $\textbf V[Z|X]=1-\sum_{j=1}^p\sigma_{ZX_j}^2$.
    \item It does not make too much sense to have a pasting stage, since we will be adding the same $\alpha$ we just peeled in portions of $\alpha/(2p)$ at each side. However, a possible way to add this whole stage is to look for the dimension that maximizes the conditional mean, once a portion of probability $\alpha/2$ have been added to each side of the selected dimension. All this, of course, provided that this maximal conditional mean be higher than the one already found during the peeling stage. If this stage is applied as described, the final region will be a rectangular centered box.
\end{enumerate}

Points 1, 2 and 3 can be stated as follows:

\begin{lem}\label{PRIMStandardLemma}
    Assume $Z \sim N(0,1)$ and $X \sim N(\textbf 0, \textbf I)$. Let us iterate $t$ times Algorithm \ref{AlgofastPRIM}. Then the whole algorithm can be reduced to a single stage of finding a centralized box with vertices located at the coordinates corresponding to the quantiles $\frac{1}{2}(\beta_T)^{1/p}$ and $1-\frac{1}{2}(\beta_T)^{1/p}$ of each of the $p$ variables.
\end{lem}

\subsection{fastPRIM and Principal Components}\label{fastPRIMPC}

Note that if $Z \sim N(\mu,\sigma^2)$ and $X \sim N(\textbf 0, \Sigma)$, the same algorithm as in Section \ref{SecfastPRIM} can be used. The only difference is that the final box will be a rectangular Lebesgue set, not necessarily a square as before (although it continues being a square in probability). Some comments are in order.

First, with each of the variables having possible different variances, we are also peeling the random variables with lower variance. That is, we are peeling precisely the variables that we do not want to touch. The whole idea behind PRIM, however, is to peel from the variables with high variance, leaving the ones with lower variance as untouched as possible. The obvious solution is to use a PCA to project on the variables with higher variance, peel on those variables, and after the box is obtained to add the whole set of variables we chose not to touch. Adding to the notation developed in Section \ref{SecPRIMPC} for PCA, call $Y'$ the projection of $Y$ to its firsts $p'$ principal components, where $0 < p'\leq p$. Algorithm \ref{AlgofastPRIMPCA} below makes this explicit.

\begin{algorithm}[!ht]
    \caption{fastPRIM with Principal Components}
    \label{AlgofastPRIMPCA}
    \begin{itemize}
        \item (PCA) Apply PCA to $X$ to obtain the space $\mathfrak X'(p')$.
        \item (Peeling) Make $2p'$ peels corresponding to each side of the box, each one with probability $\alpha(2p')^{-1}$. After $L$ steps, the centered 		box has $\beta$ measure, and its marginals will have probability $\beta^{1/p'}$ each.
        \item (Covering) Call $B_M(k)$ the box found after the $k$-th step, $k=1,\ldots,t$, of this modified peeling stage. Setting $S(Y')=S^{(1)}(Y')$, take 		the space $S^{(k)}(Y') := S^{(1)}(Y') \setminus \bigcup_{1 \leq b \leq k-1} B_M(b)$ and repeat on it the peeling stage.
        \item (Completing) The final box will be given by $\left[\mathfrak X'(p)\setminus\mathfrak X'(p')\right]\cup S^{(t)}(Y')$. That is, to the final box we are 		adding the whole subspace which we chose not to peel.
    \end{itemize}
\end{algorithm}

In this way, we avoid to select for peeling the variables with lower variance. Concededly, we are still peeling the same amount (we are getting squares, not rectangles, in probability), but we are also getting an important simplification in algorithmic complexity cost. Besides this fact, most of the comments in Section \ref{SecPRIMStandard} are still valid but one clarification has to be made: The covariance matrix of $(Z\; Y')$ has size $(p'+1)\times(p'+1)$; as before, all the non-diagonal elements are zero, except possibly the ones in the first row and the first column. Call $\sigma_{ZY_1'}, \ldots,\sigma_{ZY_p'}$. Then $\textbf E[Z|Y']=\sum_{j=1}^{p'}\sigma_{ZY_j'}\lambda_j^{-1}Y_j'$ and $\textbf {Var}[Z|Y']=\sigma_Z^2-\sum_{j=1}^{p'}\lambda_j^{-1}\sigma_{ZY_j'}^2$, where $Y_j'$ is the $j$-th component of the random vector $Y'$.

As before, we can state the following lemma:

\begin{lem}
	Assume $Z \sim N(\mu,\sigma^2)$ and $X \sim N(\textbf 0,\Sigma)$. Iterate $t$ times the covering stage of Algorithm \ref{AlgofastPRIMPCA}. 		Then the whole algorithm can be reduced to a two-stage setting: First,  to find a centralized box with vertices located at the coordinates 			corresponding to the quantiles $\frac{1}{2}(\beta_T)^{1/p'}$ and $1-\frac{1}{2}(\beta_T)^{1/p'}$ of each of the $p'$ variables. Second, add the $p-	p'$ dimensions left untouched to the final box.
\end{lem}

\begin{remark}
	Even though we have developed the algorithm with $p'\leq p$, it is not wise to try to reduce the dimensions of the input. To be sure, the rotation of 	the input in the direction of the principal components is a useful thing to do in learning settings, as \citet{DiazRaoDazard2014} have showed. 		However, \citet{Cox1968}, \citet{HadiLing1998}, and \citet{Joliffe1982}, have warned against the reduction of dimensionality.
\end{remark}

\subsection{fastPRIM and Data}\label{fastPRIMdata}

The usefulness of the previous result can be more easily seen when, for relatively large $n$, we consider the iid vectors $X_1,\ldots,X_n$ with finite second moment, since in this way we can approximate to a normal distribution by the Multivariate Central Limit Theorem:

Call $X=[X_1\cdots X_n] $ and let us assume that $n\gg 0$. By the multivariate central limit theorem, if the vectors of observations are iid, such that their distribution has mean $\mu_X$ and variance $\Sigma_X$, we can approximate $ X^* := n^{1/2}\left(\overline X-\mu_X\right)$ to a $p$-variate normal distribution with parameters $\textbf 0$ and $\Sigma_X$. That is,  $\overline X$ can be approximated to a distribution $N(\mu,(1/n)\Sigma_X)$. Now, $Y^*=X^*G$ is the PC transformation of $X^*$, where $G$ is the matrix of eigenvectors of $S$, the sample covariance matrix of $X^*$; i.e., $S = GLG^T$, and $L$ is the diagonal matrix of eigenvalues of $S$, with $l_{j^{'}}\geq l_j$ for all $j^{'}<j$.

As before, call $Y'$ the projection of $Y$ to its firsts $p'$ principal components. Apply Algorithm \ref{AlgofastPRIMPCA}.

Note that the use of the CLT is indeed well justified: since the asymptotic mean of $X^*$ is $\textbf 0$, its asymptotic mode is also at $\textbf 0$ (or around $\textbf 0$).

\subsection{Graphical Illustrations}\label{Graphical}

In the following simulations, we first test PRIM and fastPRIM and illustrate graphically how fastPRIM compares to PRIM either in the input space $\mathfrak X(p)$ or in the PC space $\mathfrak X'(p)$. We generated a synthetic dataset derived from a simulation setup similar to the one used in Section \ref{Intro}, although with a single target distribution and a continuous normal response, without noise. Thus, the data $X$ was simulated as $X \sim N_p(0,\Sigma)$ with response $Z \sim N(\mu,\sigma^2)$. To control the amount of variance for each input variable and their correlations, the sample covariance matrix $\Sigma$ was constructed from a specified sample correlation matrix $R$ and sample variance matrix $V$ such that $\Sigma := V^{\sfrac{1}{2}} R V^{\sfrac{1}{2}}$, after ensuring that the resulting matrix $\Sigma$ is symmetric positive definite.

Simulations were carried out with a continuous normal response with parameters $\mu = 1$ and $\sigma = 0.2$, a fixed sample size $n = 10^3$, and no added noise (i.e. mixing weight $w = 1$). Here, we limited ourselves to a low dimensional space ($p = p' = 2$) for graphical visualization purposes. Simulations were for a fixed peeling quantile $\alpha$, a fixed minimal box support $\beta$, a fixed maximal coverage parameter $t$, and no pasting for PRIM. Empirical results presented in Figure \ref{figure02} show the marked computational efficiency of fastPRIM compared to PRIM. CPU times are plotted against PRIM and fastPRIM coverage parameters $k \in \{1, \ldots, t\}$ and $t \in \{1, \ldots, 20\}$, respectively, in the original input space $\mathfrak X(2)$ and PC space $\mathfrak X'(2)$.

\begin{figure}[!hbt]
	\centering
    	\includegraphics[width=0.8\textwidth]{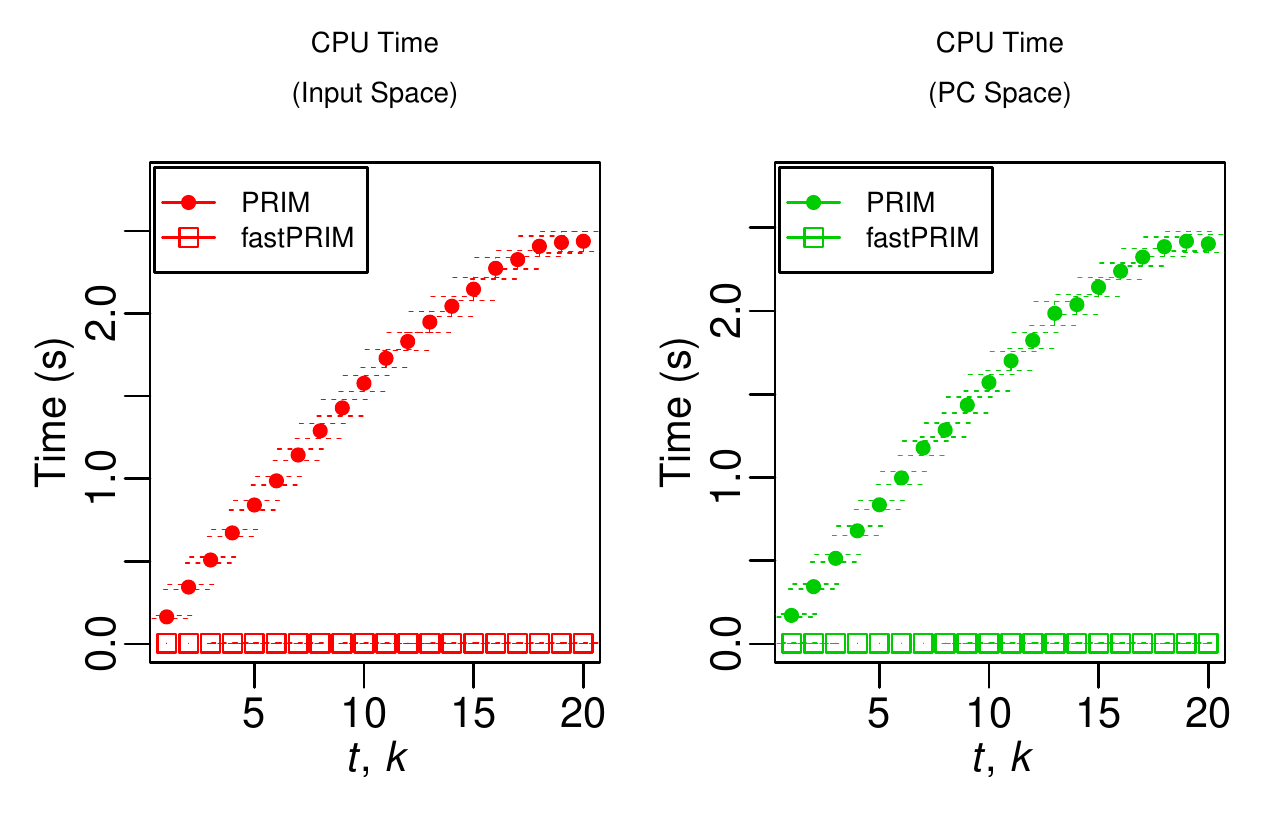}
    	\vskip -10pt
    	\caption{Total CPU time as a function of coverage. For all plots, comparison of speed metrics are reported against coverage parameter $k \in \{1, 		\ldots, t\}$ for PRIM and coverage parameter $t \in \{1, \ldots, 20\}$ for fastPRIM, in the original input space $\mathfrak X(2)$ (left), and the 		PC space $\mathfrak X'(2)$ (right) for each algorithm. Total CPU time in seconds (s). Mean estimates and standard errors of the means are 		reported after generating 128 Monte-Carlo replicates.}
    	\label{figure02}
\end{figure}

Further, empirical results presented in Figure \ref{figure03} show PRIM and fastPRIM box coverage sequences as a function of PRIM and fastPRIM coverage parameters $k \in \{1, \ldots, t\}$ and $t \in \{1, \ldots, 20\}$, respectively. Notice the centering and nesting of the series of fastPRIM boxes in contrast to the sequence of boxes induced by PRIM (Figure \ref{figure03}).

\begin{figure}[!hbt]
	\centering
    	\includegraphics[width=1.0\textwidth]{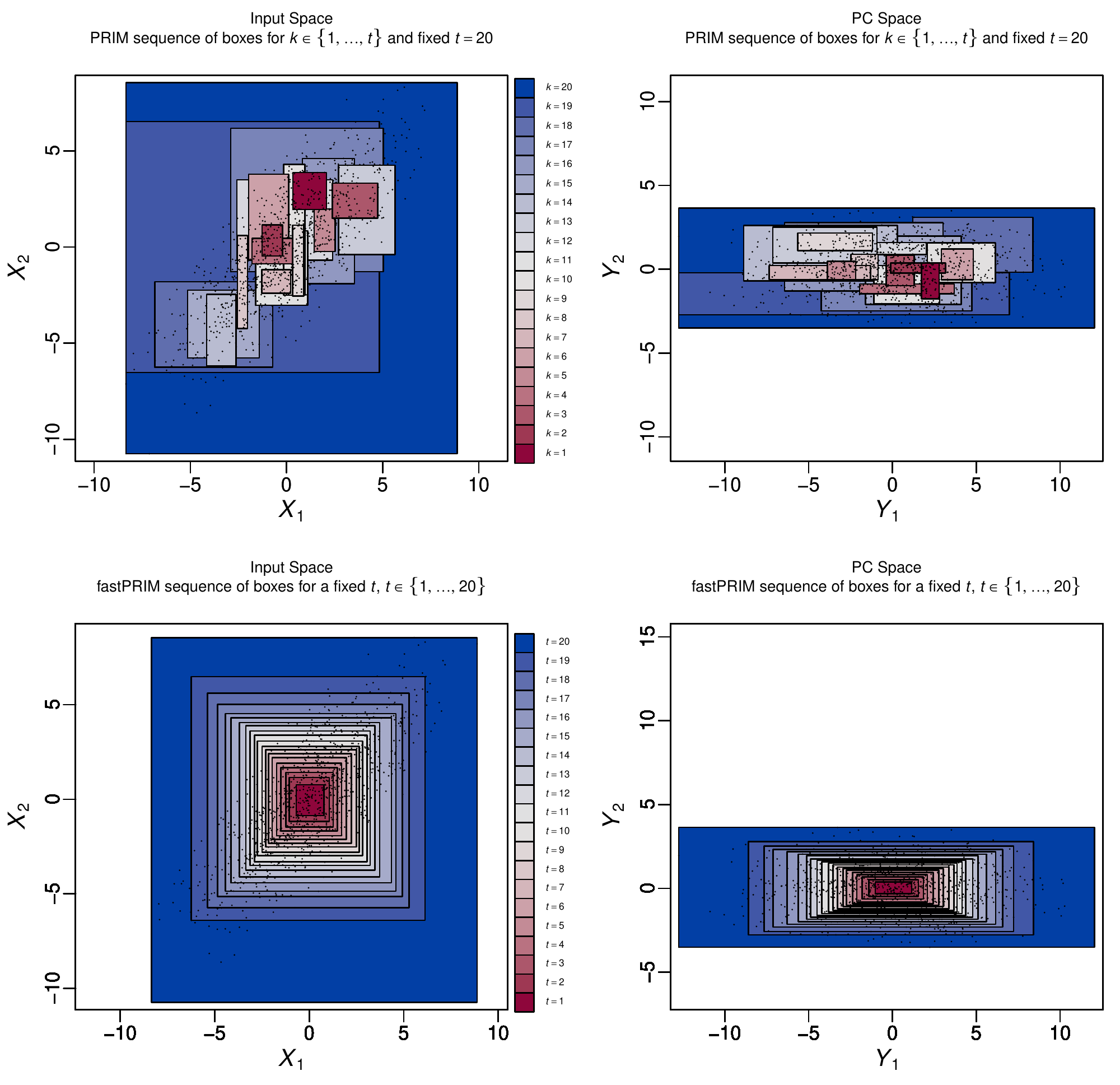}
    	\vskip -10pt
    	\caption{PRIM and fastPRIM box coverage sequences. Top row: PRIM complete sequence of coverage boxes, each corresponding to a coverage 		step $k \in \{1, \ldots, t\}$ with a fixed peeling quantile $\alpha = 0.05$, and a fixed maximal coverage parameter $t =20$, corresponding to a 		fixed minimal box support $\beta = 0.05$. Bottom row: fastPRIM complete sequence of coverage boxes, each corresponding to a fixed 			coverage parameter $t \in \{1, \ldots, 20\}$, with a fixed $\beta = 0.05$. Results are given in the input space $\mathfrak X(2)$ (left) and in the 		PC space $\mathfrak X'(2)$ (right). The red to blue palette corresponds to a range of box output means from the largest to 	the smallest, 		respectively.}
    \label{figure03}
\end{figure}
\vskip 0pt

\section{Comparison of the Algorithms in the Input and PC Spaces}\label{SpacesCompared}

The greatest theoretical advantage of fastPRIM is that, because of the centrality of the boxes, it gives us a framework to compare the output mean in the original input space and in the PC space, something that cannot be attained with the original PRIM algorithm in which the behaviour of the final region is unknown (see Figure \ref{figure02}). \citet{PolonikWang2010} explain how PRIM tries to approximate regression level curves, an objective that the algorithm does not accomplish in general. With the idea of level curves in mind, it is clear that the bump of a multivariate normal distribution can be seen as the data inside the ellipsoids of concentration. This concept is the key to prove the optimality of the box found on the PC space. By optimality here we mean the box with minimal Lebesgue measure among all possible central boxes found by fastPRIM with probability measure $\beta$.

\begin{lem}\label{Circum}
    Let $E$ be a $p$-dimensional ellipsoid. The rectangular box that is circumscribing $E$ (i.e. centered at the center of $E$, with sides parallel to the axes of $E$, such that each of its edges is of length equal to the axis length of $E$ in the corresponding dimension), is the box with the minimal volume of all the rectangular boxes containing $E$.
\end{lem}

The proof of Lemma \ref{Circum} is well-known and is omitted here.

\begin{prop}\label{minBox}
	Let $X \sim N(\textbf 0,\Sigma)$. Assume that the true bump $E$ of $X$ has probability measure $\beta' > 0$. Then, it is possible to find a 		rectangular box $R$ by fastPRIM that circumscribes $E$ under the PC rotation with minimal Lebesgue measure over all rectangular boxes 		containing $E$ and the set of all possible rotations.
\end{prop}

\begin{proof}
	The true bump satisfies that $\textbf P[x \in E]=\beta'$. This bump, by definition of normality, lives inside an ellipsoid of concentration $E$, of 		volume $\text{Vol}(E)=\pi_p\prod_{1\leq j\leq p}r_j$, where $r_j$ is the length of the semi axis of the dimension $j$ and $\pi_p$ is a constant that 	only depends on the dimension $p$. By Lemma \ref{Circum} above, the box $R$ with sides parallel to the axes of $E$, and circumscribing $E$, 	has minimal volume over all the boxes containing $E$ and its volume is $2^p\prod_{1\leq j\leq p}r_j$, and $2^p>\pi_p$. Let us assume that $\textbf 	P[x \in R] = \beta$ (thus $\beta' < \beta$).

	Note now that $R$ is parallel to the axes in the space of principal components $\mathfrak X'(p)$  and it is centered at its origin. Therefore, 		provided an appropriate small $\alpha$ (it is possible that we need to adjust proportionally $\alpha$ on each direction of the principal components 	to obtain the box that circumscribes $E$), the minimal rectangular box $R$ containing the bump $E$ can be approximated through 				fastPRIM and is in the direction of the principal components. As such, then the box $R$ has smaller Lebesgue measure than any other 			approximation in every other rotation.
\end{proof}

\begin{remark}
	The box of size $\beta$ circumscribing the ellipsoid of concentration $E$ is identical to $B^*_\beta$ in equation (\ref{OptBox}).
\end{remark}

Proposition \ref{minBox} allows us to compare box estimates in the PC space of PRIM (Figure \ref{figure02}, top-right) versus fastPRIM (Figure \ref{figure02}, down-right). Remember from Equation (\ref{FinalRegion}) that $R_\rho(p,1)$ is the box obtained with PRIM after a single stage of coverage. We now restrict ourselves to the case of $R_\rho(p,1)$ in the direction of the principal components (i.e., its sides are parallel to the axes of $\mathfrak X'(p)$). We establish the following result:

\begin{teo}\label{main}
 	Assume $X \sim N(0,\Sigma)$ and $Z \sim N(0,\sigma^2)$. Call $R$ the final fastPRIM box resulting from Algorithm \ref{AlgofastPRIMPCA} and 	assume $p'=p$. As in (\ref{FinalRegion}), call also $\hat R_\rho(p,1)$ the final box from Algorithm \ref{AlgoPRIM} after one stage of coverage. 		Assume that $R$ and $R_\rho(p,1)$ contain the true bump. Then
	\begin{align}
		\frac{M(R)}{Vol(R)}>\frac{M(R_\rho(p,1))}{Vol(R_\rho(p,1))},
	\end{align}
	that is, the volume-adjusted box output mean of the mode of $Z$ given $R$ is bigger than the volume-adjusted box output mean of the mode of $Z	$ given $R_\rho(p)$.
\end{teo}

\begin{proof}
	Note that by definition, the two boxes have sides parallel to the axes of $\mathfrak X'(p)$. The proof is direct because of the assumptions. By 		Proposition \ref{minBox}, $R$ is the minimal box of measure $\beta$ that contains the true bump. Therefore, any other box $R'$ with parallel sides 	to $R$ that contains the bump also contains $R$. Since $R$ is centered around the mean of $Z$, every point $z$ in the support of $Z$ such that 	$z\in R'\setminus R$ have less density than $\arg\min_zf_Z(z)$. Therefore $M(R)>M(R')$. From Proposition 1 we also get that $Vol(R)<Vol(R')$.
	
	Since $R_\rho(p,1)$ is but a particular case of a box $R'$, the result follows.
\end{proof}

Not only $R$ has better volume-adjusted output mean than $R_\rho(p,1)$. We conclude showing the optimality of the latter over the former in terms of bias and variance.

\begin{teo}\label{ComparingPRIMAndFastPRIM}
	Assume $Z \sim N(\mu,\sigma^2)$ and $X \sim N(\textbf 0,\Sigma)$. Define $E$ as the true bump, and let us assume that both $R$ and $R_		\rho(p)$ cover $E$. Then $\textbf {Var}(Z|Y \in R) < \textbf {Var}(Z|Y \in R_\rho(p))$, and $R$ is unbiased while $R_\rho(p)$ is not.
\end{teo}

\begin{proof}
	Note that $R$ and $R_\rho(p,1)$ are estimators of $B_\beta^*$, as defined in Equation (\ref{OptBox}). Algorithm \ref{AlgofastPRIMPCA} is 		producing unbiased boxes since by construction it is centered around the mean. In fact, $R$ would be unbiased even if not taken in the direction of 	the PC. On the other hand, $\hat R_\rho(p)$ is almost surely biased, even in the direction of the principal components, since it is producing boxes 	that are not centered around the mean.
	
    	Now, the inequality $\textbf {Var}(Z|Y \in R) < \textbf {Var}(Z|Y \in R_\rho(p))$ stems from the fact that $R$ is the box with minimal volume 			containing $E$. Since $R$ is in the direction of the principal components, every other box that contains $E$ in the same direction also contains $R	$, in particular $R \subseteq R_\rho(p)$.
\end{proof}

\section{Simulations}\label{Simulations}

Next, we illustrate how the optimality of the box encapsulating the true bump is improved in the PC space $\mathfrak X'(p)$ as compared to the input space $\mathfrak X(p)$. Empirical results presented in Figure \ref{figure04} are for the same simulation design and the same fastPRIM and PRIM parameters as described in Subsection \ref{Graphical}, except that we now allow for higher dimensionality since no graphical visualization is desired here ($p =100$).

\begin{figure}[!hbt]
	\centering
    	\includegraphics[width=1.0\textwidth]{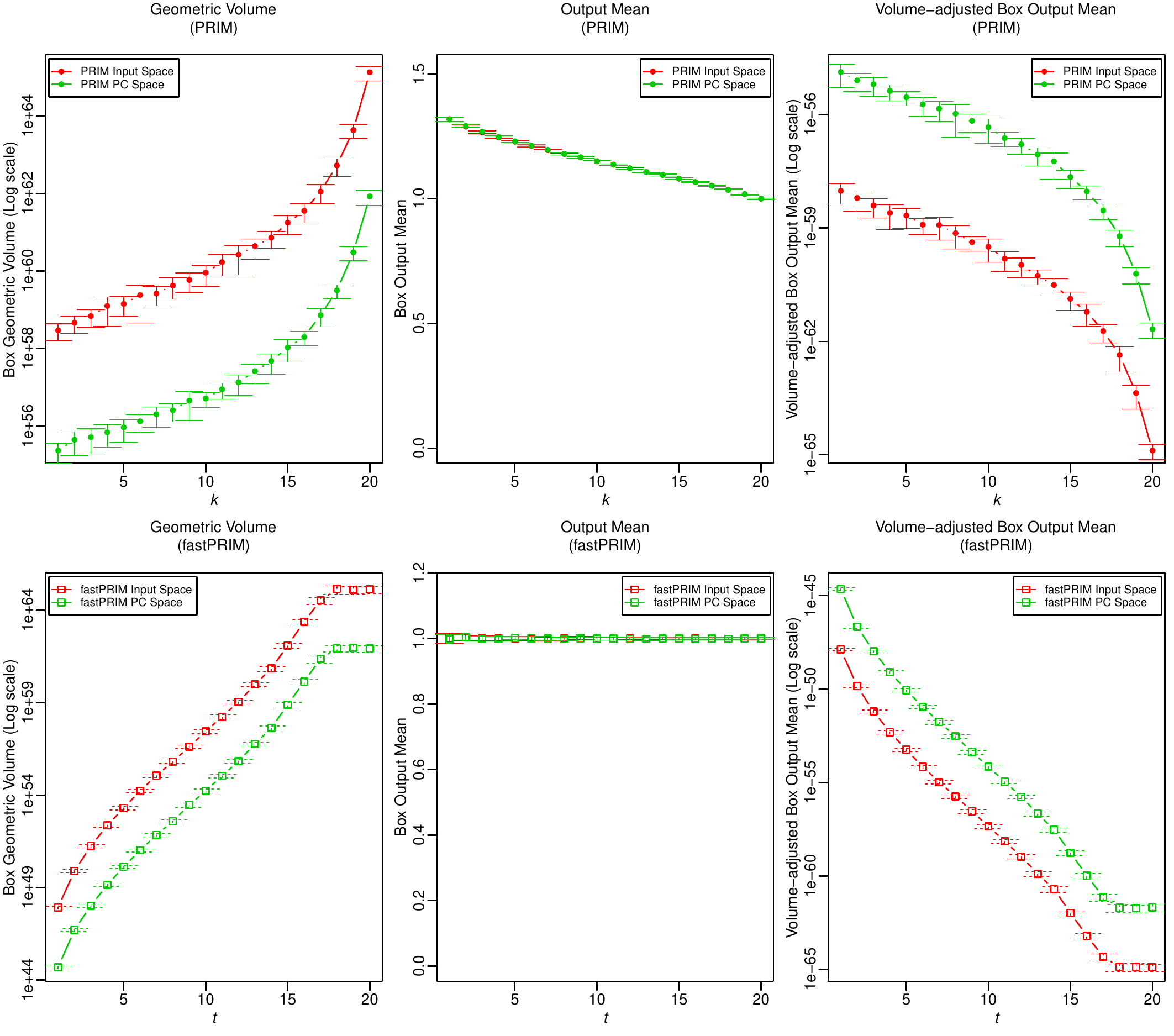}
    	\vskip -10pt
    	\caption{Box statistics and performance metrics as a function of coverage. For all plots, results are plotted against PRIM coverage parameter $k \in 	\{1, \ldots, t\}$ and fastPRIM coverage parameter $t \in \{1, \ldots, 20\}$ in the original input space $\mathfrak X(100)$ (red) vs. the PC space $		\mathfrak X'(100)$ (green), that is for $p = p' = 100$, for each algorithm: PRIM (top row) vs. fastPRIM (bottom row). First column: box geometric 	volume (Log scale); second column: box output (response) mean; third column: volume-adjusted box output (response) mean (Log scale). See 	simulation design for details and metrics definitions. Mean estimates and standard errors of the means are reported after generating 128 Monte-	Carlo replicates.}\label{figure04}
\end{figure}

Some of the theoretical results between the original input space and the PC space are borne out based on the empirical conclusions plotted in Figure \ref{figure04}. In sum, for situations with no added noise, one observes for both algorithms that: i) the effect of PCA rotation dramatically decreases the box geometric volume; ii) the box output (response) means are almost identical in the PC space and in the original input space; and iii) the volume-adjusted box output (response) means are markedly larger in the PC space than in the original input space - indicating a much more concentrated determination of the true bump structure (Figure \ref{figure04}).

Some additional comments:

\begin{enumerate}
	\item As each algorithm covers the space (up to step $k = t$), the box support and the box geometric volume are expected to increase 				monotonically (up to sampling variability) for both algorithms.
	\item The boxes are equivalent for the mean of $Z$ and the mode of $Z$ because $Z$ is normal, we expect the fastPRIM box being centered 			around the mean and therefore the conditional mean of $Z$ should be 1 (because in this simulation the mean of $Z$ is 1). While, the box for 		$Z$ given PRIM must have a different conditional expectation. This justifies the fact of looking at the mode of $Z$ inside the boxes, and not 		directly the mode of $Z$.
	\item Since the the box output (response) mean is almost perfectly constant at 1 for fastPRIM and close to 1 for PRIM, it is expected that the box 		volume-adjusted output mean decreases monotonically at the rate of the box geometric volume for both algorithms.
    	\item Also, as coverage $k,t$ increases, the two boxes $R$ and $R_\rho(p)$ of each algorithm converge to each other (covering most of the 			space), so it is expected that the output (response) means inside the final boxes converge to each other as well (i.e. towards the whole space 		mean response 1).
\end{enumerate}

To illustrate the effect of increasing dimensionality, we plot in Figure \ref{figure05} the profiles of gains in volume-adjusted box output (response) mean as a function of increasing dimensionality $p \in \{2,3,\ldots,8,9,10,20,30,\ldots,180,190,200\}$. Here, the gain is measured in terms of a ratio of the quantity of interest in the PC space $\mathfrak X'(p')$ over that in the original input space $\mathfrak X(p)$. Empirical results presented are for the same simulation design and the same fastPRIM and PRIM parameters as described in subsection \ref{Graphical}. Notice the extremely fast increase in volume-adjusted box output (response) mean ratio as a function of dimensionality $p$, that is, the marked larger value of volume-adjusted box output (response) mean in the PC space as compared to the one in the input space for both algorithms. Notice also the weak dependency with respect to the coverage parameters ($k,t$).\\

\begin{figure}[!hbt]
    \centering
    \includegraphics[width=0.8\textwidth]{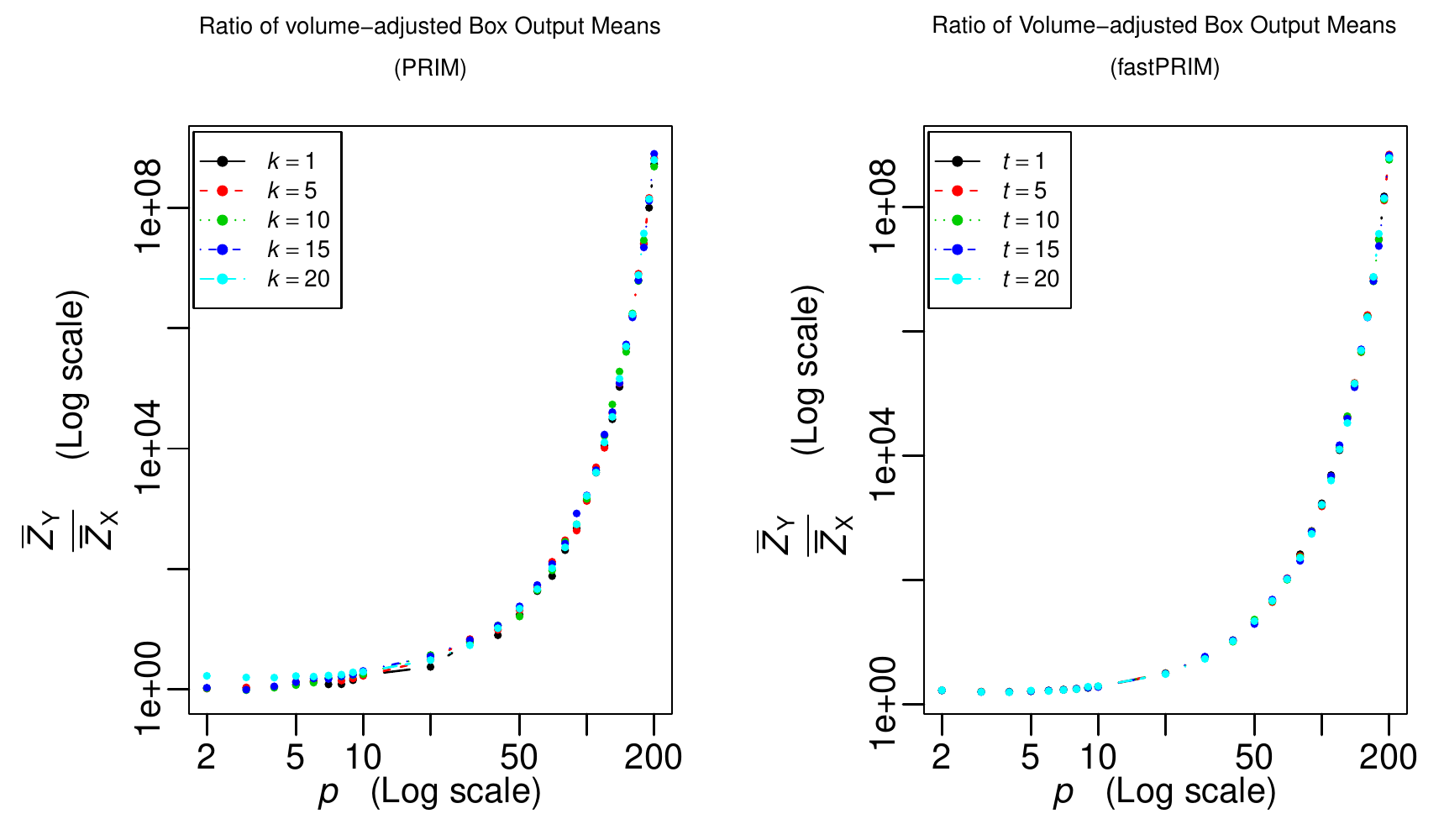}
    \vskip -15pt
    \caption{Gains profiles in volume-adjusted box output (response) mean as a function of dimensionality $p$. For all plots, comparison of box statistics and performance metrics profiles are reported as a ratio of the values obtained in the PC space $\mathfrak X'(p')$ (denoted Y) over the original input space $\mathfrak X(p)$ (denoted X). We show empirical results for varying dimensionality $p \in \{2,3,\ldots,8,9,10,20,30,\ldots,180,190,200\}$, a range of PRIM and fastPRIM coverage parameters ($k,t \in \{1, 5, 10, 15, 20\}$), and for both algorithms: PRIM (left) vs. fastPRIM (right). Both coordinate axes are on the log scale.}
    \label{figure05}
\end{figure}

Further, using the same simulation design and the same fastPRIM and PRIM parameters as described in subsection \ref{Graphical}, we compared the efficiency of box estimates generated by both algorithms in the PC space $\mathfrak X'(p')$ as a function of dimension $p'$ and coverage parameters $k,t$ for PRIM or fastPRIM, respectively. Notice, the reduced box geometric volume (Figure \ref{figure06}) and increased box volume-adjusted output (response) mean (Figure \ref{figure07}) of fastPRIM as compared to PRIM.

\begin{figure}[!hbt]
    \centering
    \includegraphics[width=0.8\textwidth]{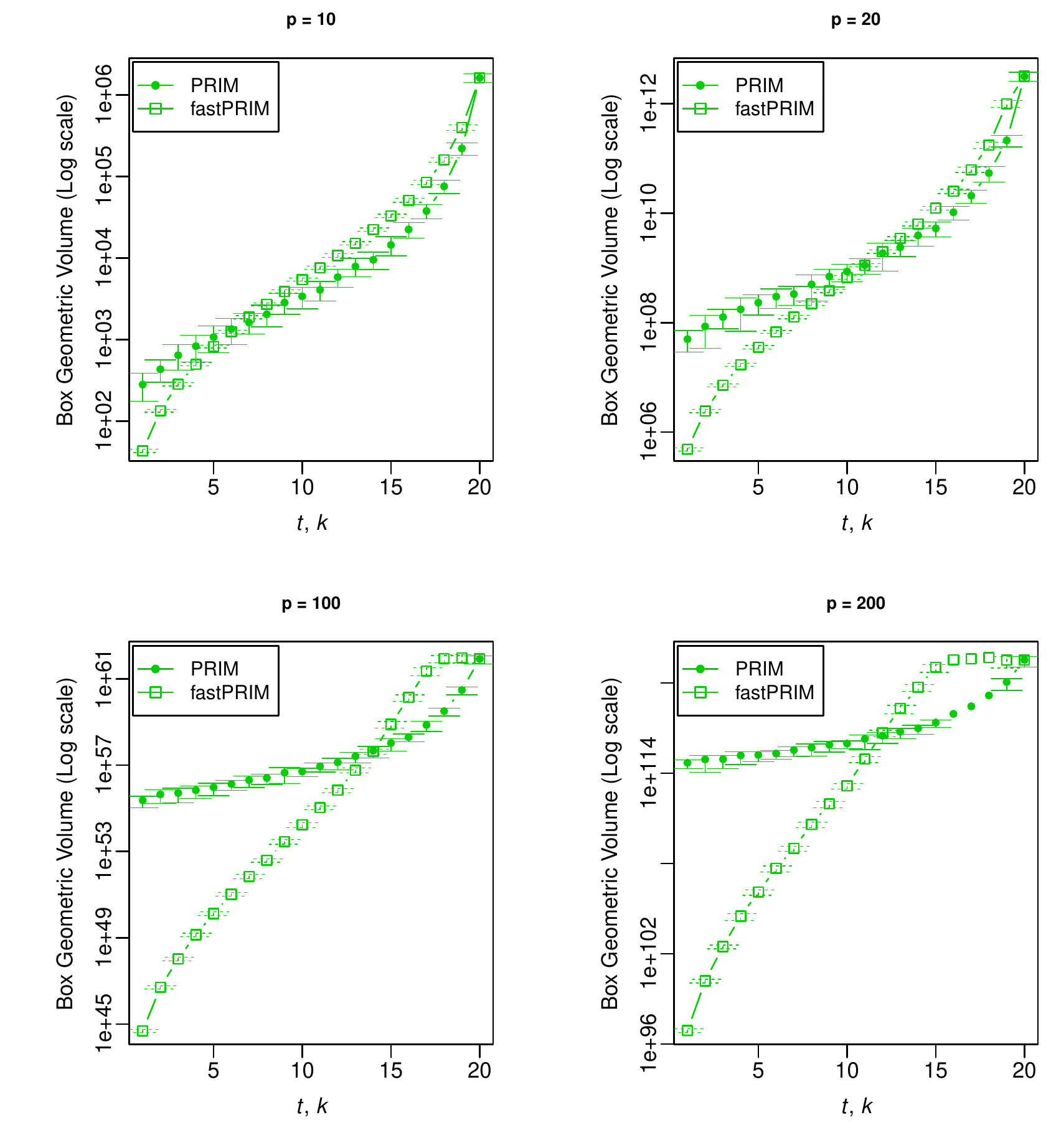}
    \vskip -15pt
    \caption{Comparative profiles of box geometric volumes in the PC space $\mathfrak X'(p')$ as a function of dimension $p'$ and coverage parameters $k \in \{1, \ldots, t\}$ or $t \in \{1, \ldots, 20\}$ for PRIM or fastPRIM, respectively. We show results for a range of dimension $p' \in \{10,20,100,200\}$ and a range of PRIM and fastPRIM coverage parameters $k \in \{1, \ldots, t\}$ or $t \in \{1, \ldots, 20\}$. The 'y' axes are on the Log scale.}
    \label{figure06}
\end{figure}

\begin{figure}[!hbt]
    \centering
    \includegraphics[width=0.8\textwidth]{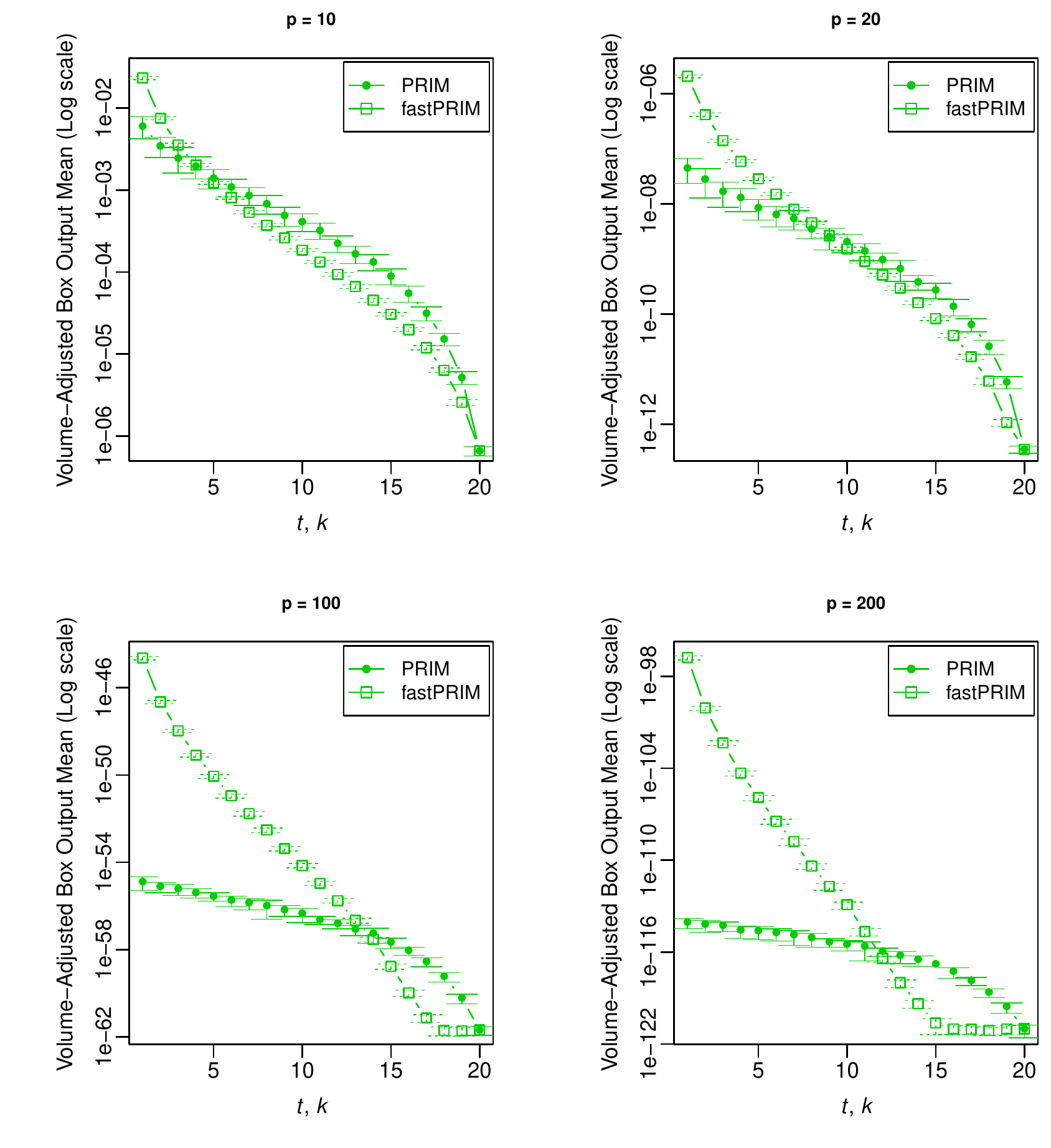}
    \vskip -15pt
    \caption{Comparative profiles of box volume-adjusted output (response) means in the PC space $\mathfrak X'(p')$ as a function of dimension $p'$ and coverage parameters $k \in \{1, \ldots, t\}$ or $t \in \{1, \ldots, 20\}$ for PRIM or fastPRIM for PRIM and fastPRIM, respectively. We show results for a range of dimension $p' \in \{10,20,100,200\}$ and a range of PRIM and fastPRIM coverage parameters $k \in \{1, \ldots, t\}$ or $t \in \{1, \ldots, 20\}$. The 'y' axes are on the Log scale.}
    \label{figure07}
\end{figure}

Finally, in Figures \ref{figure08} and \ref{figure09} below we compare variances of fastPRIM and PRIM volume-adjusted box output (response) means in the PC space $\mathfrak X'(p')$ as a function of dimension $p'$ and coverage parameters $k,t$ for PRIM or fastPRIM, respectively. Empirical results are presented for the same simulation design and the same fastPRIM and PRIM parameters as described in subsection \ref{Graphical}. Results show that the variance of fastPRIM box geometric volume (Figure \ref{figure08}) is reduced than its PRIM counterparts for coverage $t$ not too large ($ \leq 10-15$), which is matched to a reduced variance of fastPRIM \emph{volume-adjusted} box output (response) mean for coverage $t$ not too small ($ \leq 10-15$).

\begin{figure}[!hbt]
    \centering
    \includegraphics[width=0.8\textwidth]{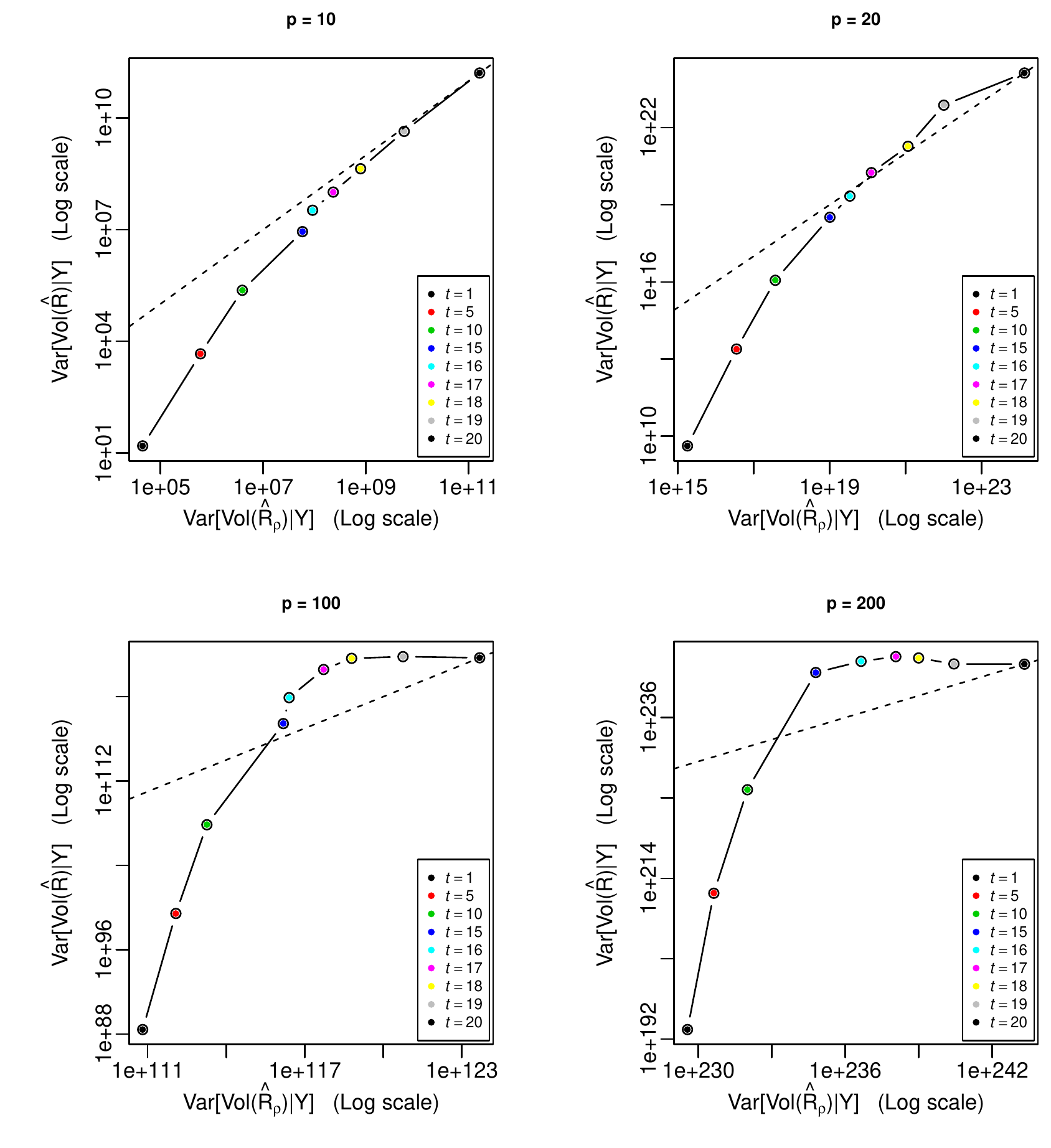}
    \vskip -15pt
    \caption{Comparative profiles of variances of box geometric volumes in the PC space $\mathfrak X'(p')$ as a function of dimensionality $p'$ and coverage parameters $k \in \{1, \ldots, t\}$ or $t \in \{1, \ldots, 20\}$ for PRIM or fastPRIM, respectively. In all subplots, we show the variances of box geometric volumes of both algorithms against each other for a range of PRIM and fastPRIM coverage parameters ($k,t \in \{1, 5, 10, 15,16,17,18,19 20\}$) in four dimensions $p' \in \{10,20,100,200\}$. The identity (doted) line is plotted. All axes are on the Log scale.}
    \label{figure08}
\end{figure}

\begin{figure}[!hbt]
    \centering
    \includegraphics[width=0.8\textwidth]{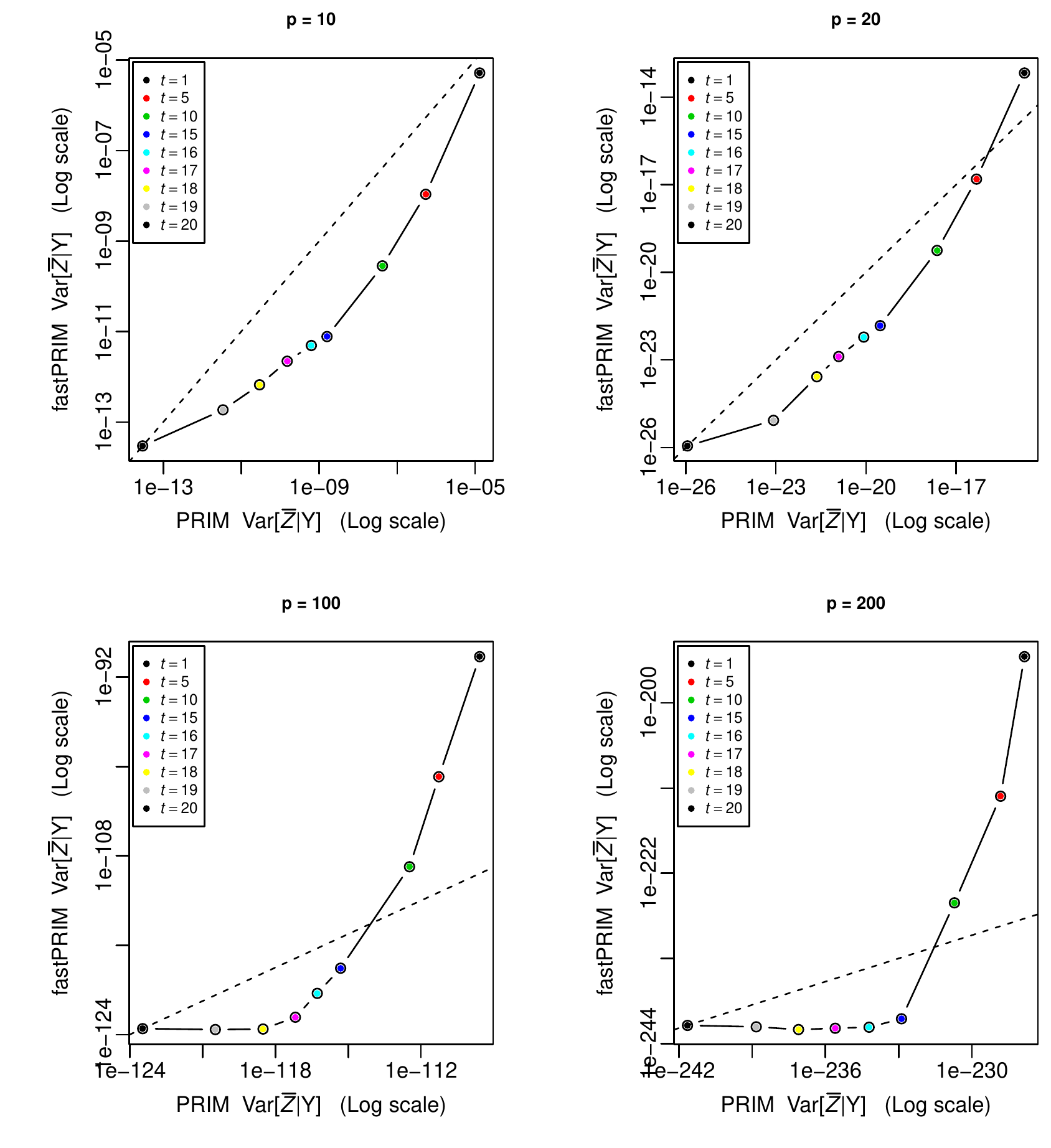}
    \vskip -15pt
    \caption{Comparative profiles of variances of box volume-adjusted output (response) means in the PC space $\mathfrak X'(p')$ as a function of dimensionality $p'$ and coverage parameters $k \in \{1, \ldots, t\}$ or $t \in \{1, \ldots, 20\}$ for PRIM or fastPRIM, respectively. In all subplots, we show the variances of the volume-adjusted box output (response) means of both algorithms against each other for a range of PRIM and fastPRIM coverage parameters ($k,t \in \{1, 5, 10, 15,16,17,18,19 20\}$) in four dimensions $p' \in \{10,20,100,200\}$. The identity (doted) line is plotted. All axes are on the Log scale.}
    \label{figure09}
\end{figure}

Of note, the results in Figures \ref{figure06} and \ref{figure07} below, and similarly in \ref{figure08} and \ref{figure09}, are for the sample size $n = 1000$ of this simulation design. In particular, efficiency results of fastPRIM versus PRIM box estimates show some dependency with respect to coverage parameters $k,t$ for large coverages and increasing dimensionality. As discussed above, this reflects a finite sample-effect favoring PRIM box estimates in these coverages and dimensionality.

Notice finally in Figures \ref{figure06} and \ref{figure07} how the curves approach each other for the largest coverage step $k = t = 20$, and similarly in \ref{figure08} and \ref{figure09} how the curves approach the identity line. This is in line with the aforementioned convergence point of the two boxes $R$ and $R_\rho(p)$ as coverage increases.

\newpage
\ 
\newpage
\
\newpage
\
\newpage

\section{Discussion}

Our analysis here corroborates what \citet{DiazRaoDazard2014} have showed on how the rotation of the input space to the one of principal components is a reasonable thing to do when modeling a response-predictor relationship.  In fact, \citet{DazardRao2010} use a \emph{sparse} PC rotation for improving bump hunting in the context of high dimensional genomic predictors.  And \citet{DazardRaoMarkowitz2012} also show how this technique can be applied to find additional heterogeneity in terms of survival outcomes for colon cancer patients. The geometrical analysis we present here shows that as long as the principal components are not being selected prior to modeling the response, then these improved variables can produce more accurate mode characterizations.  In order to elucidate this effect, we introduced the fastPRIM algorithm, starting with a supervised learner and ending up with an unsupervised one. This analysis opens the question on whether is possible to go from supervised to unsupervised settings in more general bump hunting situations, not only modes; and more generally, whether is possible to go from unsupervised to supervised in other learning contexts beyond bump hunting.

\vskip0.2in
\noindent
{\bf Acknowledgements}:  All authors supported in part by NIH grant NCI R01-CA160593A1. We would like to thank Rob Tibshirani, Steve Marron and Hemant Ishwaran for helpful discussions of the work. This work made use of the High Performance Computing Resource in the Core Facility for Advanced Research Computing at Case Western Reserve University.

\bibpunct[,]{(}{)}{;}{a}{}{,}
{\bibliographystyle{JCGS_style}
\bibliography{MyBibliography}

\end{document}